\documentclass{llncs}
\usepackage{llncsdoc}

\newcommand{\ignore}[1]{}

\ignore{

\newtheorem{theorem}{Theorem}

\newtheorem{lemma}[theorem]{Lemma}

}

\begin{document}

\title{Non-Adaptive Learning a Hidden Hipergraph}
\author{Hasan Abasi${}^1$ \and Nader H. Bshouty${}^1$ \and Hanna Mazzawi${}^2$}
\institute{Department of Computer Science\\ Technion, Haifa, 32000 \and
Google, London.\\
76 Buckingham Palace Rd.}


%
\maketitle

\begin{abstract}
We give a new deterministic algorithm that non-adaptively learns a hidden hypergraph from edge-detecting
queries. All previous non-adaptive algorithms either run in exponential time
or have non-optimal query complexity. We give the first polynomial time
non-adaptive learning algorithm for learning hypergraph
that asks almost optimal number of queries.
\end{abstract}

\section{Introduction}
Let ${\cal G}_{s,r}$ be a set of all labeled hypergraphs of
rank at most $r$ on the set $V=\{1,2,\ldots,n\}$
with at most $s$ edges.
Given a hidden hypergraph $G\in {\cal G}_{s,r}$, we need to identify it by asking
{\it edge-detecting queries}. An edge-detecting query $Q_G(S)$, for $S\subseteq V$
is: does $S$ contain at least one edge of $G$? Our objective is to non-adaptively learn
the hypergraph $G$ by asking as few queries as possible.

This problem has many applications in chemical reactions, molecular biology and genome sequencing.
In chemical reactions, we are given a set of chemicals,
some of which react and some which do not.
When multiple chemicals are combined in one
test tube, a reaction is detectable if and
only if at least one set of the chemicals in the tube reacts.
The goal is to identify which sets react
using as few experiments as possible. The time needed
to compute which experiments to do is a
secondary consideration, though it is polynomial for
the algorithms we present.
See \cite{GK98,T99,BAK01,DVMT02,ABKRS04,MRY04,MP04,AA05,BGK05,AC06,DH06,RS07,CH08,AC08,CCFS11,CLY13,CFS14,ABM14}
for more details and many other applications in molecular biology.

In all of the above applications
the rank of the hypergraph is much smaller than the number of edges
and both are much smaller than the number of vertices~$n$.
Therefore, throughout the paper, we will assume that $r\le s$ and $s=o(n)$.

The above hypergraph learning problem is
equivalent to the
problem of non-adaptively learning a monotone DNF with at most $s$ monomials (monotone term),
where each monomial contains at most $r$ variables
($s$-term $r$-MDNF) from membership queries~\cite{A87,AC08}.
In this paper we will use the later terminology rather
than the hypergraph one.

The adaptive learnability of $s$-term $r$-MDNF was studied in~\cite{AC06,DH06,AC08,ABM14}.
In \cite{ABM14}, Abasi et. al. gave a polynomial time adaptive learning algorithm for $s$-term $r$-MDNF
with almost optimal query complexity. The non-adaptive learnability of $s$-term $r$-MDNF
was studied in~\cite{T99,MP04,MRY04,GHTW06,DH06,CLY13,BG14}.

Torney~,\cite{T99}, first introduced the problem and gave some applications in molecular biology.
The first explicit non-adaptive learning algorithm for $s$-term $r$-MDNF
was given by Gao et. al., \cite{GHTW06}.
They show that this class can be learned using $(n,(s,r))$-cover-free family ($(n,(s,r))$-CFF).
This family is a set $A\subseteq \{0,1\}^n$
of assignments such that for every distinct $i_1,\ldots,i_s,j_1,\ldots,j_r\in \{1,\ldots,n\}$
there is $a\in A$ such that $a_{i_1}=\cdots=a_{i_s}=0$ and $a_{j_1}=\cdots=a_{j_r}=1$.
Given such a set, the algorithm simply takes all the monomials $M$ of size at most $r$
that satisfy $(\forall a\in A)(M(a)=1\Rightarrow f(a)=1)$. It
is easy to see that the disjunction
of all such monomials is equivalent to the target function.
Assuming a set of $(n,(s,r))$-CFF of size $N$ can be constructed in time $T$,
this algorithm learns $s$-term $r$-MDNF with $N$ queries in time $O({n\choose r}+T)$.

In~\cite{DH06,B13}, it is shown that any set $A\subset\{0,1\}^n$
that non-adaptively learns $s$-term $r$-MDNF is an $(n,(s-1,r))$-CFF.
Therefore, the minimum size
of an $(n,(s-1,r))$-CFF is also a lower bound for
the number of queries (and therefore also for the time) for non-adaptively learning $s$-term $r$-MDNF.
It is known, \cite{SWZ00}, that any $(n,(s,r))$-CFF must have size at least $\Omega (N(s,r)\log n)$
where
\begin{eqnarray}\label{Nsr}
N(s,r)=\frac{s+r}{\log{s+r\choose r}}{s+r\choose r}.
\end{eqnarray}
Therefore, any non-adaptive algorithm for learning $s$-term $r$-DNF must ask
at least $N(s-1,r)\log n=\Omega(N(s,r)\log n)$ queries and runs in at least $\Omega(N(s,r)n$ $\log n)$ time.

Gao et. al. constructed an $(n,(s,r))$-CFF of size $S=(2s\log n/\log(s\log n))^{r+1}$
in time $\tilde O(S)$. It follows from \cite{SWZ}
that an $(n,(s,r))$-CFF of size
$O\left((sr)^{\log^* n}\log n\right)$ can be constructed in polynomial time.
A polynomial time almost optimal constructions of size $N(s,r)^{1+o(1)}\log n$
for $(n,(s,r))$-CFF
were given in \cite{B12,B14b,BG14,FLS14} which give
better query complexities, but still, the above algorithms have exponential time complexity
$O({n\choose r})$, when $r$ is not constant. The latter result implies
that there is a non-adaptive algorithm that asks $Q:=N(s,r)^{1+o(1)}\log n$ queries and runs in
exponential time
$O({n\choose r})$.
Though, when $r=O(1)$ is constant, the above algorithms
run in polynomial time and are optimal. Therefore, we will assume $r=\omega(1)$.

Chin et. al. claim in~\cite{CLY13} that they have a
polynomial time algorithm that constructs an
$(n,(s,r))$-CFF of optimal size.
Their analysis is misleading.\footnote{Some parts of the construction
can indeed be performed in polynomial time,
but not the whole construction} The size is indeed optimal but the time
complexity of the construction is $O({n\choose r+s})$.
But even if a $(n,(s,r))$-CFF can be constructed in polynomial
time, the above learning algorithm still takes $O({n\choose r})$ time.

Macula et. al., \cite{MP04,MRY04}, gave several randomized
non-adaptive algorithms. We first use their ideas combined with
the constructions of $(n,(r,s))$-CFF in~\cite{B12,B14b,BG14,FLS14} to
give a new non-adaptive algorithm that asks $N(s,r)^{1+o(1)}\log^2 n$
queries and runs in $poly(n,N(s,r))$ time. This algorithm is almost
optimal in~$s$ and $r$ but quadratic in $\log n$. We then use a new technique
that changes any non-adaptive learning algorithm that asks $Q(r,s,n)$ queries
to a non-adaptive learning algorithm that asks $(rs)^2\cdot Q(r,s,(rs)^2)\log n$
queries. This give a non-adaptive learning algorithm that asks $N(s,r)^{1+o(1)}\log n$
queries and runs in $n\log n\cdot poly(N(s,r))$ time.

The following table summarizes the results ($r=\omega(1)$)

\begin{center}
\begin{tabular}{c||c|c|}
References &Query Complexity& Time Complexity\\
\hline\hline

\cite{GHTW06}& $N(s,r)\cdot (r\log n/\log(s\log n))^{r+1}$ & ${n\choose r}$\\
\hline
\cite{CLY13} & $N(s,r)\log n$ & ${n\choose r+s}$\\
\hline
\cite{B12,B14b,BG14,FLS14} & $N(s,r)^{1+o(1)}\log n$ & ${n\choose r}$\\
\hline
Ours+\cite{MP04,MRY04}+\cite{BG14} & $N(s,r)^{1+o(1)}\log^2n$ & $poly(n,N(s,r))$\\
\hline
Ours & $N(s,r)^{1+o(1)}\log n $ & $(n\log n) \cdot poly(N(s,r))$\\
\hline
Ours,\ $r=o(s)$& $N(s,r)^{1+o(1)}\log n $ & $(n\log n) \cdot N(s,r)^{1+o(1)}$\\\hline\hline
Lower Bound~\cite{DH06} & $N(s,r)\log n $ & $(n \log n)\cdot N(s,r)$\\ \hline
\end{tabular}
\end{center}

This paper is organized as follows. Section 2 gives some definitions
and preliminary results that will be used throughout the paper.
Section 3 gives the first algorithm
that asks $N(s,r)^{1+o(1)}\log^2n$ membership queries and
runs in time $poly(n,N(s,r))$. Section 4 gives the reduction and
shows how to use it
to give the second algorithm that
asks $N(s,r)^{1+o(1)}\log n $ membership queries and runs in time $(n\log n) \cdot N(s,r)^{1+o(1)}$.
All the algorithms in this paper are deterministic.
In the full paper we will also consider randomized algorithms that
slightly improve (in the $o(1)$ of the exponent) the query and time complexity.

\section{Definitions}

\subsection{Monotone Boolean Functions}
For a vector $w$, we denote by $w_i$ the $i$th entry of $w$.
Let $\{e^{(i)}\ |\ i=1,\ldots,n\}\subset \{0,1\}^n$ be the standard basis.
That is, $e^{(i)}_j=1$ if $i=j$ and $e^{(i)}_j=0$ otherwise.
For a positive integer~$j$,
we denote by $[j]$ the set $\{1,2,\ldots,j\}$. For two assignments
$a,b\in \{0,1\}^n$ we denote by $(a\wedge b)\in\{0,1\}^n$ the bitwise AND assignment. That is,
$(a\wedge b)_i=a_i\wedge b_i$.

Let $f(x_1,x_2,\ldots,x_n)$ be a boolean
function from $\{0,1\}^n$ to $\{0,1\}$.
For $1\le i_1<i_2<\cdots<i_k\le n$
and $\sigma_1,\ldots,\sigma_k\in\{0,1\}\cup\{x_1,\ldots,x_n\}$ we denote by
$$f|_{x_{i_1}\gets \sigma_1, x_{i_2}\gets
\sigma_2,\cdots, x_{i_k}\gets\sigma_k}$$ the function $f(y_1,\ldots,y_n)$ where
$y_{i_j}=\sigma_j$ for all $j\in[k]$ and $y_i=x_i$ for all
$i\in [n]\backslash \{i_1,\ldots,i_k\}$. We say that the variable $x_i$ is {\it relevant}
in $f$ if $f|_{x_i\gets 0}\not\equiv f|_{x_i\gets 1}$. A variable $x_i$ is {\it irrelevant}
in $f$ if it is not relevant in $f$.
We say that the class is {\it closed under variable projections} if for every $f\in C$
and every two variables $x_i$ and $x_j$, $i,j\le n$, we have $f|_{x_i\gets x_j}\in C$.

For two assignments $a,b\in \{0,1\}^n$, we write $a\le b$ if
for every $i\in [n]$, $a_i\le b_i$. A Boolean function $f:\{0,1\}^n\to\{0,1\}$ is {\it monotone} if
for every two assignments $a,b\in\{0,1\}^n$, if $a\le b$ then $f(a)\le f(b)$.
Recall that every monotone boolean function $f$ has a unique representation
as a reduced monotone DNF,~\cite{A87}. That is, $f = M_1\vee M_2 \vee \cdots \vee M_s$ where
each {\it monomial} $M_i$ is an ANDs of input variables, and for every
monomial $M_i$ there is a unique assignment $a^{(i)}\in\{0,1\}^n$ such that $f(a^{(i)})=1$
and for every $j\in [n]$ where $a^{(i)}_j=1$ we have $f(a^{(i)}|_{x_j\gets 0})=0$. We call
such assignment a {\it minterm} of the function $f$. Notice that
every monotone DNF can be uniquely determined by its minterms~\cite{A87}.
That is, $a\in\{0,1\}^n$ is a minterm of $f$ iff $M:=\wedge_{i\in\{j:a_j=1\}}x_i$ is a
monomial in $f$.

An {\it $s$-term $r$-MDNF} is a monotone DNF with at most $s$ monomials,
where each monomial contains at most $r$ variables. It is easy to see
that the class $s$-term $r$-MDNF is closed under variable projections.

\subsection{Learning from Membership Queries}
Consider a {\it teacher} that has a {\it target function}
$f:\{0,1\}^n\to \{0,1\}$ that is $s$-term $r$-MDNF. The teacher
can answer {\it membership queries}.
That is, when receiving $a\in\{0,1\}^n$ it returns $f(a)$.
A {\it learning algorithm} is an algorithm
that can ask the teacher membership queries.
The goal of the learning algorithm is to
{\it exactly learn} (exactly find) $f$ with minimum number of
membership queries and optimal time complexity.

Let $c$ and $H\supset C$ be classes of boolean formulas.
We say that $C$ is {\it learnable from} $H$
in time $T(n)$ with $Q(n)$ membership queries
if there is a learning algorithm that, for
a target function $f\in C$, runs in time $T(n)$,
asks at most $Q(n)$ membership queries and outputs a function $h$ in $H$
that is equivalent to $C$. When $H=C$ then we say
that $C$ is {\it properly learnable}
in time $T(n)$ with $Q(n)$ membership queries.

In adaptive algorithms the queries can depend on the answers to the previous
queries where in non-adaptive algorithms the queries are independent
of the answers to the previous queries and therefore all the queries
can be asked in parallel, that is, in one step.

\subsection{Learning a Hypergraph}
Let ${\cal G}_{s,r}$ be a set of all labeled
hypergraphs on the set of vertices $V=\{1,2,\ldots,n\}$
with $s$ edges of rank (size) at most $r$.
Given a hidden hypergraph $G\in {\cal G}_{s,r}$, we need to identify it by asking
{\it edge-detecting queries}. An edge-detecting query $Q_G(S)$, for $S\subseteq V$
is: does $S$ contain at least one edge of $G$? Our objective is to learn (identify)
the hypergraph $G$ by asking as few queries as possible.

This problem is equivalent to learning $s$-term $r$-MDNF $f$ from
membership queries. Each edge $e$ in the hypergraph corresponds to
the monotone term $\wedge_{i\in e}x_i$ in $f$ and the edge-detecting query $Q_G(S)$
corresponds to asking membership queries
of the assignment $a^{(S)}$ where $a^{(S)}_i=1$ if and only if $i\in S$.
Therefore, the class ${\cal G}_{s,r}$ can be regarded as the set of
$s$-term $r$-MDNF. The class of $s$-term $r$-MDNF
is denoted by ${\cal G}^*_{s,r}$.
Now it obvious that any learning algorithm for ${\cal G}^*_{s,r}$ is also
a learning algorithm for ${\cal G}_{s,r}$.

The following example shows that we cannot allow two edges $e_1\subset e_2$.
Let $G_1$ be a graph where $V_1=\{1,2\}$ and $E_1=\{\{1\},\{1,2\}\}$.
This graph corresponds
to the function $f=x_1\vee x_1x_2$ that is equivalent to $x_1$
which corresponds to the graph $G_2$ where $V_2=\{1,2\}$ and $E_2=\{\{1\}\}$.
Also, no edge-detecting query can distinguish between $G_1$ and $G_2$.

We say that $A\subseteq \{0,1\}$ is an {\it identity testing set}
for ${\cal G}^*_{s,r}$
if for every two distinct $s$-term $r$-MDNF $f_1$ and $f_2$
there is $a\in A$ such that $f_1(a)\not= f_2(a)$. Obviously, every
identity testing set for ${\cal G}^*_{s,r}$ can be used as queries to non-adaptively
learns ${\cal G}^*_{s,r}$.

\subsection{Cover Free Families}\label{CVF}

An $(n,(s,r))$-cover free
family ($(n,(s,r))$-CFF), \cite{KS64}, is a set $A\subseteq \{0,1\}^n$ such that for every $1\le
i_1< i_2<\cdots < i_d\le n$ where $d = s +r$ and every $J
\subseteq [d]$ of size $|J|=s$ there is $a\in A$ such that
$a_{i_k} = 0$ for all $k \in J$ and $a_{i_j} = 1$ for all $j
\in [d]\backslash J$. Denote by $N(n,(s,r))$ the minimum size of such set.
Again here we assume that $r\le s$ and $s=o(n)$.
The lower bound in
\cite{SWZ00,MW04} is
\begin{eqnarray}\label{LBCFF}
N(n,(s,r))\ge \Omega\left(N(s,r)\cdot\log n\right)
\end{eqnarray}
where $N(s,r)$ is as defined in (\ref{Nsr}).
It is known that a set of random
\begin{eqnarray}\label{rand2}
m&=&O\left(r^{1.5}\left(\log \left(\frac{s}{r}+1\right)\right)
\left (N(s,r)\cdot\log n+\frac{N(s,r)}{s+r}\log\frac{1}{\delta}\right)\right)\nonumber\\
&=& N(s,r)^{1+o(1)}(\log n+\log(1/\delta))
\end{eqnarray}
assignments $a^{(i)}\in \{0,1\}^n$, where each $a^{(i)}_j$ is $1$ with
probability $r/(s+r)$, is an $(n,(s,r))$-CFF with probability at least $1-\delta$.

It follows from \cite{B12,B14b,BG14,FLS14} that there is
a polynomial time (in the size of the CFF)
deterministic construction of $(n,(s,r))$-CFF
of size
\begin{eqnarray}\label{recent}
N(s,r)^{1+o(1)}\log n\end{eqnarray} where the $o(1)$ is with respect to $r$.
When $r=o(s)$ the construction runs in linear time~\cite{B14b,BG14}.

\subsection{Perfect Hash Function}

Let $H$ be a family of functions $h:[n]\to [q]$.
For $d\le q$ we say that $H$ is an $(n,q,d)$-{\it perfect hash family} ($(n,q,d)$-PHF)
\cite{AMS06} if for every
subset $S\subseteq [n]$ of size $|S|=d$ there is a {\it hash
function} $h\in H$ such that $h|_S$ is injective (one-to-one) on~$S$, i.e.,
$|h(S)|=d$.

In \cite{B14b} Bshouty shows
\begin{lemma}\label{ThH1c}
Let $q\ge 2d^2$. There is a $(n,q,d)$-PHF of size
$$O\left(\frac{d^2\log n}{\log(q/d^2)}\right)$$
that can be constructed in time $O(qd^2n\log n/\log(q/d^2))$.
\end{lemma}

We now give the following folklore results that will be used for randomized learning algorithms
\begin{lemma} \label{RPH} Let $q>d(d-1)/2$ be any integer.
Fix any set $S\subset [n]$ of $d$ integers. Consider
$$N:=\frac{\log(1/\delta)}{\log\left(\frac{1}{1-g(q,d)}\right)}\le \frac{\log(1/\delta)}{\log \frac{2q}{d(d-1)}}$$
uniform random hash functions $h_i:[n]\to [q]$, $i=1,\ldots,N$ where
$$g(q,d):=\left(1-\frac{1}{q}\right)\left(1-\frac{2}{q}\right)\cdots \left(1-\frac{d-1}{q}\right)$$
With probability at least $1-\delta$
one of the hash functions is one-to-one on $S$.
\end{lemma}

\section{The First Algorithm}
In this section we give the first algorithm that asks $N(s,r)^{1+o(1)}\log^2 n$ queries
and runs in time $poly(n,N(s,r))$

The first algorithm is based on the ideas in~\cite{MP04,MRY04} that were used
to give a Monte Carlo randomized algorithm.

\begin{lemma}~\label{L1} Let $A$ be an $(n,(1,r))$-CFF and $B$ be an $(n,(s-1,r))$-CFF.
There is a non-adaptive proper learning algorithm for $s$-term $r$-MDNF that asks all the
queries in $A\wedge B:=\{a\wedge b\ |\ a\in A, b\in B\}$ and finds
the target function in time $|A\wedge B|\cdot n$.
\end{lemma}
\begin{proof} Let $f$ be the target function.
For every $b\in B$, let $A_b=A\wedge b:=\{a\wedge b\ |\ a\in A\}$.
Let $I_b$ be the set of all $i\in [n]$ such that $(a\wedge b)_i\ge f(a\wedge b)$
for all $a\in A$. Let $T_b:=\wedge_{i\in I_b}x_i$.
We will show that
\begin{enumerate}
\item If $T$ is a term in $f$ then there is $b\in B$ such that $T_b\equiv T$.
\item Either $T_b=\wedge_{i\in[n]}x_i$ or $T_b$ is a subterm of one of terms of $f$.
\end{enumerate}
To prove 1, let $T$ be a term in $f$ and let $b\in B$ be an assignment
that satisfies $T$ and does not satisfy the other terms. Such assignment exists
because $B$ is $(n,(s-1,r))$-CFF. Notice that $f(x\wedge b)=T(x)=T(x\wedge b)$. If $x_i$ is in $T$
and $f(a\wedge b)=1$ then $T(a\wedge b)=T(a)=f(a\wedge b)=1$ and $(a\wedge b)_i=1$. Therefore $i\in I_b$
and $x_i$ in $T_b$.
If $x_i$ not in $T$ then since $A$ is $(n,(1,r))$-CFF
there is $a'\in A$ such that $T(a')=1$ and $a'_i=0$. Then $(a'\wedge b)_i=0$
where $f(a'\wedge b)=1$. Therefore $i$ is not in $I_b$ and
$x_i$ is not in $T_b$. Thus, $T_b\equiv T$.

We now prove 2. We have shown in 1 that if $b$ satisfies one term $T$
then $T_b\equiv T$. If $b$ does not satisfy any one of the terms in $f$
then $f(a\wedge b)=0$ for all $a\in A$ and then $T_b=\wedge_{i\in[n]}x_i$.
Now suppose $b$ satisfies at least two terms $T_1$ and $T_2$.
Consider any variable $x_i$. If $x_i$ not in $T_1$ then as before
$x_i$ will not be in $T_b$. This shows that $T_b$ is a subterm of $T_1$.\qed
\end{proof}

This gives the following algorithm
\begin{figure}[h!]
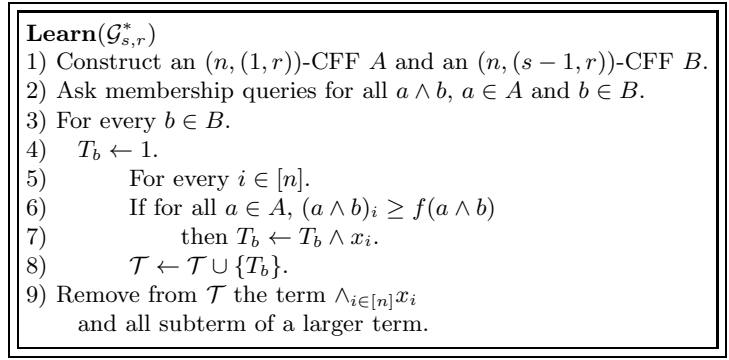

  \begin{center}
  \fbox{\fbox{\begin{minipage}{28em}
  \begin{tabbing}
  xxxx\=xxxx\=xxxx\=xxxx\= \kill
  {\bf Learn$({\cal G}^*_{s,r})$}\\
  1) Construct an $(n,(1,r))$-CFF $A$ and an $(n,(s-1,r))$-CFF $B$.\\
  2) Ask membership queries for all $a\wedge b$, $a\in A$ and $b\in B$.\\
  3) For every $b\in B$.\\
  4) \>$T_b\gets 1$.\\
  5) \>\> For every $i\in [n]$.\\
  6) \>\> If for all $a\in A$, $(a\wedge b)_i\ge f(a\wedge b)$\\
  7) \>\> \> then $T_b\gets T_b\wedge x_i$.\\
  8) \>\> ${\cal T}\gets {\cal T}\cup \{T_b\}$.\\
  9) Remove from ${\cal T}$ the term $\wedge_{i\in[n]}x_i$\\
  \> and all subterm of a larger term.
  \end{tabbing}
  \end{minipage}}}
  \end{center}
	\caption{An algorithm for learning ${\cal G}^*_{s,r}$.}
	\label{Alg}
	\end{figure}

We now have
\begin{theorem} There is a non-adaptive proper learning algorithm for $s$term $r$-MDNF that
asks $$N(s,r)^{1+o(1)}\log^2n$$ queries and runs in time $poly(n,N(s,r))$.
\end{theorem}
\begin{proof}
Constructing a $(n,(1,r))$-CFF of size $|A|=r^2\log n$ and a $(n,(s-1,r))$-CFF
of size $|B|=N(s-1,r)^{1+o(1)}\log n=N(s,r)^{1+o(1)}\log n$ takes $poly(n,N(s,r))$ time~\cite{B12,B14b,BG14,FLS14}.
By Lemma~\ref{L1}, the learning takes time $|A\wedge B|\cdot n=poly$ $(n,N(s,r))$ time.
The number of queries of the algorithm is $|A\wedge B|\le |A|\cdot |B|=N(s,r)^{1+o(1)}$ $\log^2n$.\qed
\end{proof}

\section{The Second Algorithm}
In this section we give the second algorithm

We first prove the following result
\begin{lemma}\label{DNtoDN} Let $C$ be a class of boolean function that is closed
under variable projection. Let $H$ be a class of boolean functions
and suppose there is an algorithm
that finds the relevant variables of $f\in H$ in time $R(n)$.

If $C$ is non-adaptively learnable from $H$
in time $T(n)$ with $Q(n)$ membership queries then $C$ is non-adaptively
learnable from $H$ in time $$O\left(qd^2n\log n+\frac{d^2\log n}{\log (q/d^2)}(T(q)n+R(q))\right)$$
with $$O\left(\frac{d^2Q(q)}{\log (q/d^2)}\log n\right)$$ membership queries
where $d$ is an upper bound on the number of relevant variables in $f\in C$ and $q\ge 2d^2$.
\end{lemma}
\begin{proof}
Consider the algorithm in Figure~\ref{AlgI}.
Let ${\cal A}(n)$ be a non-adaptive algorithm that learns $C$
from $H$ in time $T(n)$ with $Q(n)$ membership queries.
Let $f\in C_{n}$ be the target function.
Consider the
$(n,q,d+1)$-PHF $P$ that is constructed in Lemma~\ref{ThH1c}
(Step 1 in the algorithm). Since $C$
is closed under variable projection, for every $h\in P$
the function $f_h:=f(x_{h(1)},\ldots,x_{h(n)})$ is in $C_{q}$. Since
the membership queries to $f_h$ can be simulated by membership queries to $f$
there is a set of $|P|\cdot Q(q)$ assignments from $\{0,1\}^n$
that can be generated from ${\cal A}(q)$ that non-adaptively
learn $f_h$ for all $h\in P$ (Step 2 in the algorithm). The
algorithm ${\cal A}(q)$ learns $f_h'\in H$ that is equivalent to~$f_h$.

Then the algorithm finds the relevant variables of each $f_h'\in H$ (Step~3 in the algorithm).
Let $V_{h}$ be the set of relevant variables of $f_h'$ and let $d_{max}=\max_h |V_h|$.
Suppose $x_{i_1},\ldots,x_{i_{d'}}$, $d'\le d$ are the relevant variables
in the target function $f$. There is a map $h'\in P$ such that $h'(i_1),\ldots,h'(i_{d'})$ are distinct
and therefore $f_{h'}'$ depends on $d'$ variables. In particular, $d'=d_{max}$ (Step 4 in the algorithm).

After finding $d'=d_{max}$ we have: Every $h$ for which $f_h'$ depends on $d'$ variables necessarily satisfies
$h(i_1),\ldots,h(i_{d'})$ are distinct. Consider any other
non-relevant variable $x_j\not\in \{x_{i_1},\ldots,x_{i_{d'}}\}$. Since $P$ is
$(n,q,d+1)$-PHF, there is $h''\in P$ such that $h''(j),h''(i_1),\ldots,h''(i_{d'})$
are distinct. Then $f_{h''}'$ depends on $x_{h''(i_1)},\ldots,x_{h''(i_{d'})}$ and not in $x_{h''(j)}$.
This way the non-relevant variables can be eliminated. This is Step 6 in the algorithm.
Since the above is true for every non-relevant variable, after Step 6 in the algorithm,
the set $X$ contains only the relevant variables of $f$. Then in Steps 7 and 8, the
target function $f$ can be
recovered from any $f_{h_0}'$ that satisfies $|V(h_0)|=d'$.\qed
\end{proof}

\begin{figure}[h!]

  \begin{center}
  \fbox{\fbox{\begin{minipage}{28em}
  \begin{tabbing}
  xxxx\=xxxx\=xxxx\=xxxx\= \kill
  {\bf Algorithm Reduction I}\\
${\cal A}(n)$ is a non-adaptive learning algorithm for $C$ from $H$.\\
1) Construct an $(n,q,d+1)$-PHF $P$.\\
2) For each $h\in P$ \\
\>  Run ${\cal A}(q)$ to learn $f_h:=f(x_{h(1)},\ldots,x_{h(n)})$.\\
\>  Let $f_h'\in H$ be the output of ${\cal A}(q)$.\\
3) For each $h\in P$\\
\>  $V_h\gets$ the relevant variables in $f_h'$  \\
4) $d_{max}\gets \max_h |V_h|$.\\
5) $X\gets \{x_1,x_2,\ldots,x_n\}$.\\
6) For each $h\in P$\\
\> If $|V_h|=d_{max}$ then $X\gets X\backslash \{x_i\ |\ x_{h(i)}\not\in V_h\}$\\
7) Take any $h_0$ with $|V_{h_0}|=d_{max}$ \\
8) Replace each relevant variable $x_i$ in $f_{h_0}'$ by $x_j\in X$ where $h_0(j)=i$.\\
9) Output the function resulted in step (8).
  \end{tabbing}
  \end{minipage}}}
  \end{center}
\caption{Algorithm Reduction.}
\label{AlgI}

\end{figure}

We now prove
\begin{theorem} There is a non-adaptive proper learning algorithm for $s$-term $r$-MDNF that
asks $$N(s,r)^{1+o(1)}\log n$$ queries and runs in time $(n\log n)\cdot poly(N(s,r))$
time.
\end{theorem}
\begin{proof} We use Lemma~\ref{DNtoDN}. $C=H$ is the class of $s$-term $r$-MDNF.
This class is closed under variable projection. Given $f$ that is $s$-term $r$-MDNF,
one can find all the relevant variables in $R(n)=poly(s)$ time.
The algorithm in the previous section runs in time $T(n)=poly(n,N(s,r))$
and asks $Q(n)=N(s,r)^{1+o(1)}\log^2 n$ queries. The number of variables
in the target is bounded by $d=rs$. Let $q=3r^2s^2\ge 2d^2$. By Lemma~\ref{DNtoDN},
there is a non-adaptive algorithm that runs in time
$$O\left(qd^2n\log n+\frac{d^2\log n}{\log (q/d^2)}(T(q)n+R(q))\right)=(n\log n)poly(N(r,s))$$
and asks
$$O\left(\frac{d^2Q(q)}{\log (q/d^2)}\log n\right)=N(s,r)^{1+o(1)}\log n$$ membership queries.\qed
\end{proof}

\end{document}